\documentclass[11pt]{article}
\usepackage{amsmath,amsfonts,amssymb,amsthm}

\usepackage{fullpage}
\usepackage{hyperref}
\usepackage{bm}
\usepackage{comment}
\usepackage{booktabs}
\usepackage{multirow}
\usepackage{algorithm}
\usepackage{algorithmic}

\usepackage{amsmath}
\usepackage{amssymb}
\usepackage{mathtools}
\usepackage{amsthm}
\usepackage{hyperref}

\usepackage{algorithm}
\usepackage{algorithmic}
\usepackage[capitalize,noabbrev]{cleveref}

\newtheorem{theorem}{Theorem}

\newtheorem{lemma}{Lemma}

\newtheorem{assumption}{Assumption}

\newtheorem{remark}{Remark}

\newcommand{\R}{\mathbb{R}}

\newcommand{\E}{\mathbb{E}}

\newcommand{\KL}{\mathrm{KL}}

\newcommand{\mmse}{\mathrm{mmse}}

\newcommand{\eps}{\epsilon}

\newcommand{\norm}[1]{\left\lVert #1 \right\rVert}

\newcommand{\ip}[2]{\left\langle #1,#2\right\rangle}

\newcommand\blue {\textcolor{blue}}

\usepackage{biblatex} 
\title{When Diffusion Model Can Ignore Dimension: An Entropy-Based Theory}

\author{
  Ahmad Aghapour \thanks{Department of Mathematics, University of Michigan, Ann Arbor, 48109; email: aghapour@umich.edu. }
  \and
  Erhan Bayraktar \thanks{Department of Mathematics, University of Michigan, Ann Arbor, 48109; email: erhan@umich.edu. }
}
\date{}
\begin{document}
\maketitle

\begin{abstract}
Diffusion models perform remarkably well on high-dimensional data such as images, often using only a modest number of reverse-time steps. Despite this practical success, existing convergence theory does not fully explain why such samplers remain efficient in high dimensions. Many prior KL guarantees bound the discretization error in terms of the ambient dimension, while other improved results replace this dependence using intrinsic-dimensional or geometric structure assumptions. In this work, we develop an alternative information-theoretic perspective on diffusion sampler convergence. We prove that, for Gaussian mixture targets, the discretization error is controlled by the Shannon entropy of the latent mixture component rather than by the ambient dimension. Consequently, the leading step complexity scales linearly with latent entropy and depends only logarithmically on the second moment of the data. Our analysis also extends to discrete target distributions, where the relevant complexity is the entropy of the target rather than the dimension of the embedding space. These results suggest that diffusion sampling can remain efficient in high-dimensional spaces when the data distribution admits a compact latent representation, as is widely believed to be the case for natural images.
\end{abstract}

\section{Introduction}

Diffusion models have become one of the most successful approaches to generative modeling, achieving striking performance on high-dimensional data such as images, audio, video, and molecular structures \cite{sohl2015deep,song2019generative,song2021scorebased,dhariwal2021diffusion,rombach2022high}. These models generate samples by starting from a simple noise distribution and numerically simulating a reverse-time denoising process driven by a learned score. In practice, high-quality samples are often produced with a moderate number of reverse steps, even when the ambient dimension of the data is extremely large. This empirical success raises a basic theoretical question: why do diffusion samplers work so efficiently in high dimension?
Several recent works have attempted to explain this phenomenon by identifying
notions of effective dimension that are smaller than the ambient dimension.
One perspective is geometric: high-dimensional data may concentrate near a
low-dimensional manifold or admit small covering number, allowing diffusion
samplers to adapt to intrinsic rather than ambient dimension
\cite{liang2025low,pmlr-v291-potaptchik25a,huang2026denoising}. Another
perspective is distributional: structured classes such as Gaussian mixtures
can admit sampling guarantees whose dependence on ambient dimension is
substantially reduced or removed under appropriate assumptions
\cite{li2025dimension}. These results suggest that the apparent
high-dimensionality of natural data may not be the right complexity measure
for diffusion sampling.

In this work, we take a different information-theoretic viewpoint. Rather
than measuring complexity by ambient dimension, covering number, or geometric
intrinsic dimension, we measure the amount of latent uncertainty that the
reverse process must resolve. This viewpoint is motivated by modern image
representations. High-dimensional images are often encoded by compact latent
variables before being decoded back to pixel space. Vector-quantized and
token-based image models, for example, represent images using a sequence of
discrete codebook indices rather than raw pixels
\cite{van2017neural,esser2021taming,yu2024image}. Recent tokenizers such as
TiTok show that a $256\times256$ image can be represented using as few as a
small number of discrete latent tokens \cite{yu2024image}. Such models suggest
that, although images live in a very high-dimensional Euclidean space, much of
their generative uncertainty may be carried by a compact discrete latent
representation.

This motivates the following question: can diffusion sampler convergence be
controlled by the entropy of a latent representation, rather than by the
dimension of the ambient space? We answer this question affirmatively for
Gaussian mixture targets. We consider data of the form
\[
    Z = U + \epsilon G_0,
    \qquad
    U = \mu_J,
\]
where $J$ is a discrete latent variable indexing a mixture component and
$G_0$ is Gaussian noise. The sample $Z$ lies in $\mathbb R^d$, but the
dominant uncertainty is represented by the latent index $J$. Our main result
shows that, for a suitable score-based reverse sampler, the discretization
error is governed by the Shannon entropy $H(J)$ rather than by the ambient
dimension $d$.

More precisely, Theorem~\ref{thm:main-sampling-bound} proves that there
exists a $K$-step sampling grid such that
\[
    \KL(P_\delta\,\|\,\widehat P_{T-\delta})
    \le
    \frac{R}{2T}
    +
    \frac{4H(J)}{K}
    \left(
    2+\log_+\frac{R\,\eta_{\max}}{2H(J)}
    \right)^2
    +
    \frac12 e_{\mathrm{apx}} ,
\]
where $R=\mathbb E\|Z\|^2$, $\eta_{\max}=(\delta+\epsilon^2)^{-1}$, and
$e_{\mathrm{apx}}$ is the score-approximation contribution. Thus, apart from
initialization and score-estimation errors, the leading discretization
contribution scales as
\[
    O\!\left(
    \frac{H(J)}{K}
    \left(
    1+\log_+\frac{R\eta_{\max}}{H(J)}
    \right)^2
    \right).
\]
In particular, the leading term is linear in the entropy of the latent
component and has only logarithmic dependence on scale quantities such as the
second moment and endpoint signal-to-noise level.

The key technical idea is to rewrite the discretization error as an area
functional involving the minimum mean-square error of the Gaussian channel
associated with the latent variable. This converts a numerical analysis
problem into an information-theoretic one. 
Using the I-MMSE identity of Guo, Shamai, and Verd{\'u}~\cite{guo2005mutual} together
with the elementary bound that mutual information is at most the entropy of
the discrete latent variable, we show that the MMSE area is controlled by
$H(J)$. The resulting bound gives an entropy-based explanation for why
diffusion sampling can remain efficient in high-dimensional spaces when the
data admit a compact latent description.

Our analysis also covers discrete target distributions as the zero-variance
limit $\epsilon\to0$. In this case, the Gaussian mixture collapses to a
discrete distribution supported on embedded codewords, and the relevant
complexity becomes the entropy of the target distribution itself. This setting
captures an idealized version of diffusion over vector-quantized latent codes
or embedded discrete sequences. Such quantized representations are used in
practice in VQ-based generative models, including VQ-VAE and VQGAN-style image
tokenizers \cite{van2017neural,esser2021taming}. They are also used in latent
diffusion models: \cite{rombach2022high} train diffusion models in the latent
space of a VQ-regularized first-stage autoencoder, where the first stage uses a
vector-quantization layer and is interpreted as a VQGAN-style compression
model. In this view, the embedding or decoded image dimension may be large,
but the sampling complexity is controlled by the information content of the
latent code rather than by the latent dimension.

This Shannon-entropy dependence is important for two reasons. First, Shannon
entropy measures the effective uncertainty of the latent representation rather
than the size of the latent space. Thus it can be much smaller than
$\log |\mathcal J|$ when the distribution over latent codes is nonuniform,
as is often expected in learned discrete representations. Second, the bound
does not require the latent space to be finite: it applies to countably
infinite collections of mixture components whenever $H(J)<\infty$. In this
sense, the result depends on the information content of the latent variable,
not on the number of possible codes or on the ambient embedding dimension.
The remaining dependence on the second moment and endpoint signal-to-noise
scale is only logarithmic, so these scale parameters affect the discretization
complexity much more mildly than an ambient-dimensional term would.
\section{Related Work}

We now situate our result within the growing literature on non-asymptotic
convergence guarantees for diffusion samplers. A common goal in this line of
work is to determine whether the number of reverse steps must scale with the
ambient dimension, or whether sharper guarantees can exploit additional
structure in the target distribution. Recent theory has made substantial progress under broad data
assumptions, without requiring log-concavity or restrictive functional
inequalities
\cite{chen2023sampling,chen2023improved,benton2024nearly,li2025odt}.
These analyses typically decompose the sampling error into initialization,
discretization, and score-approximation terms. The resulting guarantees can be
roughly divided according to the metric in which the final sampling error is
measured, most commonly total variation distance or Kullback--Leibler
divergence.

\paragraph{Total-variation convergence.}
One line of work studies convergence in total variation distance. For general
target distributions, sharp recent results show that DDPM-type samplers can
achieve total-variation error with nearly linear dependence on the ambient
dimension, for example through $O(d/T)$-type rates under minimal assumptions
\cite{li2025odt}. Related results for probability-flow ODE samplers obtain
nearly $d/\varepsilon$ iteration complexity in total variation under minimal
smoothness assumptions on the data distribution \cite{li2024sharp}. More
refined TV analyses show that this ambient-dimensional dependence can improve
when the target has additional structure: DDPM and DDIM can adapt to an
intrinsic dimension $k$ for broad low-dimensional target classes
\cite{liang2025low}, and approximate Gaussian mixture targets admit
dimension-free TV guarantees, independent of both the ambient dimension and
the number of mixture components up to logarithmic factors
\cite{li2025dimension}. These results give important evidence that diffusion
samplers can exploit structure beyond ambient dimension.

\paragraph{Gaussian mixtures.}
The Gaussian-mixture result of \cite{li2025dimension} is particularly close
to our setting. That work studies targets that are close in total variation to
finite isotropic Gaussian mixtures, with bounded component means, and proves
TV guarantees whose leading complexity is independent of the ambient dimension
and the number of components up to logarithmic factors. Our work is
complementary but different in both metric and complexity measure. We work
directly in KL divergence and obtain a discretization bound controlled by the
Shannon entropy $H(J)$ of the latent mixture component. Thus the leading
quantity in our bound is the effective uncertainty of the component index,
rather than the worst-case number of mixture components. This distinction is
important when the mixture weights are highly nonuniform, since $H(J)$ can be
much smaller than $\log K$. Moreover, our formulation allows countably many
latent components, provided the latent entropy and second moment are finite.

\paragraph{Why total variation may not fully explain practice.}
The goal of convergence theory is not only to prove that diffusion samplers
converge, but also to explain the scale at which approximation errors arise in
practice. From this perspective, total variation distance may be too coarse
for some high-dimensional generative modeling questions. In many diffusion
analyses, the score-estimation contribution is naturally expressed as an
integrated squared score, noise-prediction, or denoising error. When this
quantity is converted into total variation, the resulting bound can saturate.
For example, if an image has $D\approx 2^{16}$ coordinates, then an
$\ell_2^2$ prediction error of order one corresponds to a root-mean-square
coordinate error of order $D^{-1/2}\approx 2^{-8}$. Such an error can be small
from the perspective of coordinatewise neural prediction, while the
corresponding TV upper bound may still be $O(1)$ and therefore close to
vacuous. This motivates working directly with KL divergence, where the
pathwise decomposition retains the natural squared-error approximation term.

\paragraph{KL convergence guarantees.}
A second line of work therefore studies KL convergence more directly. Early
and recent KL analyses prove convergence for score-based samplers under
finite-moment, smoothing, or score-accuracy assumptions, with rates that are
often nearly linear in the ambient dimension
\cite{chen2023improved,benton2024nearly}. Other work obtains KL guarantees
under regularity conditions such as finite relative Fisher information with
respect to a Gaussian reference measure \cite{conforti2025kl}. More recent
results show that KL convergence can also adapt to geometric low-dimensional
structure: under manifold or intrinsic-dimensional assumptions, the number of
sampling steps can scale linearly, up to logarithmic factors, with an
intrinsic dimension rather than with the ambient dimension
\cite{pmlr-v291-potaptchik25a,huang2026denoising}. These results provide a
geometric explanation for dimension reduction in KL: the target distribution
may concentrate near a low-dimensional manifold or satisfy an intrinsic
dimension condition.

\paragraph{Entropy-based KL analyses.}
Closest in spirit to our work are entropy-based KL analyses, which seek
dimension-free discretization bounds in terms of information-theoretic
quantities rather than geometric dimension. \cite{aghapour2026entropy}
derive dimension-free KL discretization guarantees using an MMSE
representation and information-theoretic control
\cite{aghapour2026entropy}. Their analysis is fully independent of
second-moment scale quantities, but the resulting entropy dependence is
quadratic: under weaker assumptions it is controlled by the square of the
order-$1/2$ R\'enyi entropy, $H_{1/2}^2$, and under an additional
concentration condition on the information content it yields a
Shannon-entropy-squared bound.

The rest of the paper develops this latent-entropy viewpoint. We first define
the Gaussian-mixture target model and the corresponding score-based reverse
sampler. We then decompose the pathwise KL error into initialization,
discretization, and score-approximation terms. Finally, we prove the main
entropy-based discretization bound by expressing the discretization error as
an MMSE area functional for the latent Gaussian channel and controlling this
area using the I-MMSE identity.

\section{Problem Setup and Sampling Procedure}

Let $p$ be a target distribution on $\R^d$. We begin with the standard Brownian noising process
\[
    dX_t=dW_t,\qquad X_0\sim p,\qquad t\in[0,T],
\]
where $(W_t)_{t\in[0,T]}$ is a Brownian motion independent of $X_0$. We write $Z:=X_0$ and denote by $p_t$ the law of $X_t$. For $t>0$, let
\[
    m_t(x):=\E[Z\mid X_t=x]
\]
be the Bayes denoiser. Since $X_t=Z+\sqrt t\,G$ with $G\sim\mathcal N(0,I_d)$, Tweedie's formula gives
\[
    \nabla\log p_t(x)=\frac{m_t(x)-x}{t}.
\]
Thus the reverse-time process $Y_s:=X_{T-s}$ satisfies, under the usual time-reversal conditions \cite{haussmann1986time},
\[
    dY_s
    =
    \frac{m_{T-s}(Y_s)-Y_s}{T-s}\,ds+dB_s,
    \qquad
    Y_0=X_T .
\]
We analyze the reverse dynamics up to time $T-\delta$, for a fixed $\delta>0$, in order to avoid the singular terminal endpoint. Let
\[
    0=s_0<s_1<\cdots<s_K=T-\delta,
    \qquad
    t_k:=T-s_k .
\]

A useful way to view the stochastic DDIM sampler with $\eta=1$ in this Brownian setting is as follows. On each interval $(s_{k-1},s_k]$, the posterior mean is frozen at the previous grid point, while the linear part of the reverse drift is integrated exactly. If the exact denoiser were available, this would correspond to
\[
    d\widetilde Y_s
    =
    \frac{
    m_{t_{k-1}}(\widetilde Y_{s_{k-1}})-\widetilde Y_s
    }{T-s}\,ds+dB_s,
    \qquad
    s\in(s_{k-1},s_k].
\]
Solving this linear SDE over one interval gives
\[
    \widetilde Y_{s_k}
    =
    \frac{t_k}{t_{k-1}}\widetilde Y_{s_{k-1}}
    +
    \left(1-\frac{t_k}{t_{k-1}}\right)
    m_{t_{k-1}}(\widetilde Y_{s_{k-1}})
    +
    \sqrt{t_k\left(1-\frac{t_k}{t_{k-1}}\right)}\,\xi_k,
\]
where $\xi_k\sim\mathcal N(0,I_d)$ are independent. In this form, the discretization error is precisely the error caused by holding the posterior mean fixed during each reverse-time step.

We now specialize to the structured targets considered in this paper. Assume that the data are generated from a latent Gaussian mixture:
\[
    Z=U+\eps G_0,
    \qquad
    U=\mu_J,
    \qquad
    G_0\sim\mathcal N(0,I_d),
\]
where $J$ is a discrete latent variable with probabilities $(\pi_j)$, the vectors $\mu_j\in\R^d$ are mixture centers, and $G_0$ is independent of $J$. For $\eps>0$, the target distribution is the continuous Gaussian mixture
\[
    p(z)=\sum_j \pi_j\,\varphi_{\eps^2 I_d}(z-\mu_j).
\]

Under this model, the forward process admits the representation
\[
    X_t=U+\sqrt{t+\eps^2}\,G,
    \qquad
    G\sim\mathcal N(0,I_d).
\]
It is therefore natural to track the posterior mean of the latent center,
\[
    M_t(x):=\E[U\mid X_t=x].
\]
For Gaussian mixtures, the score can be written as
\[
    \nabla\log p_t(x)
    =
    \frac{M_t(x)-x}{t+\eps^2},
    \qquad
    M_t(x)=x+(t+\eps^2)\nabla\log p_t(x).
\]
Thus the latent posterior mean is not an extra oracle: it can be recovered directly from the score.

The reverse SDE can therefore be written in latent-posterior form,
\[
    dY_s
    =
    \frac{M_{T-s}(Y_s)-Y_s}{T-s+\eps^2}\,ds+dB_s,
    \qquad
    Y_0=X_T .
\]
This suggests a modification of the usual posterior-mean freezing rule. Instead of freezing the posterior mean of the full data variable, we freeze the posterior mean of the latent center. Given a learned score $\widehat s_t$, define
\[
    \widehat M_t(x)
    :=
    x+(t+\eps^2)\widehat s_t(x).
\]
The proposed sampler is the continuous-time interpolation
\begin{equation}
\label{eq:latent-sampler-sde}
    d\widetilde Y_s
    =
    \frac{
    \widehat M_{t_{k-1}}(\widetilde Y_{s_{k-1}})
    -
    \widetilde Y_s
    }{T-s+\eps^2}\,ds+dB_s,
    \qquad
    s\in(s_{k-1},s_k].
\end{equation}
Equivalently, setting
\[
    a_k:=\frac{t_k+\eps^2}{t_{k-1}+\eps^2},
\]
one obtains the one-step update
\[
    \widetilde Y_{s_k}
    =
    a_k\widetilde Y_{s_{k-1}}
    +
    (1-a_k)\widehat M_{t_{k-1}}(\widetilde Y_{s_{k-1}})
    +
    \sqrt{(t_k+\eps^2)(1-a_k)}\,\xi_k,
\]
where $\xi_k\sim\mathcal N(0,I_d)$ are independent. Substituting the definition of $\widehat M_t$ shows that the sampler only requires evaluations of the learned score at the grid points. The latent posterior mean is used to reveal the structure of the dynamics, but the implemented algorithm is fully score-based.

\section{Analysis of the Proposed Sampler}

We now analyze the sampler defined in \eqref{eq:latent-sampler-sde}. The purpose of this section is to separate the sampling error into the error coming from time discretization and the error coming from score approximation. The entropy-based control of the discretization term is developed in the next section.

\begin{assumption}[Gaussian mixture target]
\label{ass:gmm-target}
The target distribution is generated as
\[
    Z=U+\eps G_0,
    \qquad
    U=\mu_J,
    \qquad
    G_0\sim\mathcal N(0,I_d),
\]
where $J$ is a discrete latent variable taking values in an index set $\mathcal J$ with probabilities $(\pi_j)_{j\in\mathcal J}$, the vectors $\{\mu_j\}_{j\in\mathcal J}\subset\R^d$ are the mixture centers, and $G_0$ is independent of $J$. We assume that the latent entropy is finite,
\[
    H(J):=-\sum_{j\in\mathcal J}\pi_j\log\pi_j<\infty,
\]
and that the target has finite second moment,
\[
    \E\norm{Z}^2 = R < \infty .
\]
\end{assumption}

\begin{remark}
When $\eps=0$, the model reduces to a discrete distribution supported on the
set of centers $\{\mu_j\}$. This includes latent diffusion models with
vector-quantized latent representations, where data are represented by
discrete codebook indices, mapped to Euclidean latent embeddings, and the
diffusion process is performed in this latent space before decoding back to
data space \cite{rombach2022high}. It also connects to continuous diffusion
language models, where discrete tokens are embedded into Euclidean space and
the Gaussian diffusion process is performed on these continuous embeddings.
For example, Diffusion-LM denoises a sequence of Gaussian vectors into word
vectors \cite{li2022diffusion}.
\end{remark}
\begin{remark}
When $\eps>0$, we model the data distribution in pixel space as the decoded
image associated with a discrete latent code, plus a small continuous residual.
Concretely, the center $\mu_j$ may be interpreted as the pixel-space image
obtained by decoding the discrete latent code $j$, while the Gaussian term
$\eps G_0$ represents reconstruction error, decoder variability, or additional
observation noise around that decoded image. Thus the model
\[
    Z=\mu_J+\eps G_0
\]
can be viewed as a pixel-space Gaussian mixture induced by a discrete latent
representation.

This perspective is motivated by modern image tokenizers, which represent
high-dimensional images using compact discrete latent codes before decoding
them back to pixel space. For example, TiTok represents a $256\times256$ image
using as few as $32$ discrete tokens \cite{yu2024image}. In this setting, the
sample $Z$ lies in the ambient pixel space $\R^d$, but the dominant generative
uncertainty is carried by the discrete latent index $J$, while $\eps$ controls
the size of the reconstruction or residual error around the decoded center.
\end{remark}

Let $P$ and $\widetilde P$ denote the path laws on $C([0,T-\delta],\R^d)$ of the exact latent-posterior reverse process and the approximate process in \eqref{eq:latent-sampler-sde}, initialized from the same law at time zero. On each interval $(s_{k-1},s_k]$, define the drift mismatch along the exact path by
\begin{equation}
\label{eq:drift-mismatch}
    \Delta_s
    :=
    \frac{
    M_{T-s}(Y_s)
    -
    \widehat M_{t_{k-1}}(Y_{s_{k-1}})
    }{
    T-s+\eps^2
    } .
\end{equation}
The remaining source of error comes from replacing the exact score by the
learned score along the sampling grid. To isolate this contribution, we
measure the error of the learned model through the induced \(x_0\)-prediction
loss at the grid points. This quantity is the part of the final sampling error
that depends on the learned score, rather than on the discretization of the
reverse dynamics. We therefore record an assumption that controls the total
score-approximation error along the chosen sampling schedule.
As in the schedule-dependent score-error assumption of
~\cite{benton2024nearly}, the approximation assumption below is
understood to be imposed on the same discretization grid used by the sampler.
In particular, when applying the assumption in Theorem~\ref{thm:main-sampling-bound}, the
grid is the final \(K\)-step grid constructed in that theorem.
\begin{assumption}
\label{ass:approx-error}
For the chosen SNR grid $\gamma_k=1/t_k$, the learned score satisfies
\[
    \mathcal E_{\mathrm{apx}}
    =
    \sum_{k=1}^K
    (\gamma_k-\gamma_{k-1})
    L_{x_0}(\gamma_{k-1})
    \le e_{\mathrm{apx}},
\]
where
\[
    L_{x_0}(\gamma_k)
    :=
    \E\left[
    \norm{
    m_{t_k}(X_{t_k})-\widehat m_{t_k}(X_{t_k})
    }^2
    \right],
    \qquad
    t_k=\gamma_k^{-1}.
\]
Here $e_{\mathrm{apx}}\ge 0$ denotes the total $x_0$-prediction error of the learned model along the sampling schedule.
\end{assumption}
\begin{theorem}
\label{thm:main-sampling-bound}
Assume Assumption~\ref{ass:gmm-target}, and suppose that \(H(J)>0\) and
\(R>0\). Let the sampler \eqref{eq:latent-sampler-sde} be initialized from the
Gaussian prior \(\mathcal N(0,T I_d)\), and let
\(\widehat P_{T-\delta}\) denote the law of its output at time \(T-\delta\).
Let \(P_\delta\) denote the law of the forward process \(X_\delta\). Define
\[
    \eta_{\max}:=\frac{1}{\delta+\eps^2},
    \qquad
    \alpha:=\frac{2H(J)}{R},
    \qquad
    \ell:=\log_+\!\left(\frac{\eta_{\max}}{\alpha}\right)
    =
    \log_+\!\left(\frac{R\,\eta_{\max}}{2H(J)}\right),
\]
where \(\log_+(x):=\max\{\log x,0\}\). Then, for every integer
\[
    K\ge 4(2+\ell),
\]
there exists a \(K\)-step grid \(\Gamma_K\) for the sampler
\eqref{eq:latent-sampler-sde} such that, if
Assumption~\ref{ass:approx-error} holds on this grid with approximation bound
\(e_{\mathrm{apx}}\), then
\[
    \KL(P_\delta\,\|\,\widehat P_{T-\delta})
    \le
    \frac{R}{2T}
    +
    \frac{4H(J)}{K}
    \left(2+\ell\right)^2
    +
    \frac12 e_{\mathrm{apx}} .
\]
Equivalently,
\[
    \KL(P_\delta\,\|\,\widehat P_{T-\delta})
    \le
    \frac{R}{2T}
    +
    \frac{4H(J)}{K}
    \left(
    2+\log_+\frac{R\,\eta_{\max}}{2H(J)}
    \right)^2
    +
    \frac12 e_{\mathrm{apx}} .
\]
In particular, the discretization contribution to the KL error is of order
\[
    \frac{H(J)}{K}
    \left(
    1+\log_+\frac{R\,\eta_{\max}}{H(J)}
    \right)^2 .
\]
Consequently, up to initialization and score-approximation errors, it is
sufficient to take
\[
    K
    =
    O\!\left(
    \frac{H(J)}{\xi}
    \left(
    1+\log_+\frac{R\,\eta_{\max}}{H(J)}
    \right)^2
    \right)
\]
steps to make the discretization contribution at most \(\xi\).
\end{theorem}
\paragraph{Proof sketch.}
The full proof of Theorem~\ref{thm:main-sampling-bound} is deferred to
Appendix~\ref{app:proof-main}. By the pathwise-KL argument in
Appendix~\ref{app:pathwise-kl}, the total error is reduced to an initialization
term, a score-approximation term, and a discretization term. The initialization
term is handled directly, and the score-approximation term is controlled by
Assumption~\ref{ass:approx-error}. The main task is therefore to bound the
discretization term. For this, we first rewrite the discretization error as an
MMSE area functional in the regularized SNR variable
\[
    \eta := \frac{1}{t+\eps^2},
\]
as stated in Lemma~\ref{lem:mmse-representation}. To control this area, we need
a pointwise upper bound on the latent MMSE. Lemma~\ref{lem:entropy-mmse}
provides exactly this bound:
\[
    \mmse(\eta)\le \min\left\{R,\frac{2H(J)}{\eta}\right\}.
\]
where 
\[
    \mmse(\eta)
    =
    \E\!\left[
    \norm{
    U-\E[U\mid U+\eta^{-1/2}G]
    }^2
    \right]
\]
This is the key structural input in the proof. It shows that the MMSE is
bounded by the second moment \(R\) in the low-SNR regime and by the entropy
term \(2H(J)/\eta\) in the high-SNR regime. This two-regime behavior is what
motivates the hybrid grid construction used in the appendix.

\begin{lemma}
\label{lem:mmse-representation}
For any grid written in the regularized SNR variable as
\[
    \eta_k:=\frac1{t_k+\eps^2},
\]
the discretization term satisfies
\[
    \mathcal E_{\mathrm{disc}}
    =
    \sum_{k=1}^K
    \int_{\eta_{k-1}}^{\eta_k}
    \left(
    \mmse(\eta_{k-1})-\mmse(\eta)
    \right)\,d\eta .
\]
\end{lemma}

\begin{proof}
The proof is given in Appendix~\ref{app:lemma1proof}. This MMSE-functional
representation is closely related to the entropy-based viewpoint developed in
\cite{aghapour2026entropy}.
\end{proof}

\begin{lemma}
\label{lem:entropy-mmse}
Under Assumption~\ref{ass:gmm-target}, for every \(\eta>0\),
\[
    \mmse(\eta)
    \le
    \min\left\{
    R,\frac{2H(J)}{\eta}
    \right\}.
\]
\end{lemma}

\begin{proof}
Since \(Z=U+\eps G_0\) and \(G_0\) is independent of \(U\),
Assumption~\ref{ass:gmm-target} implies
\[
    \E\norm{U}^2\le \E\norm{Z}^2=R.
\]
Therefore
$\mmse(\eta)\le \E\norm{U}^2 \le R.$

For the second bound, we use the I-MMSE identity for the Gaussian channel:
\[
    I(U;U+\eta^{-1/2}G)
    =
    \frac12
    \int_0^\eta \mmse(\lambda)\,d\lambda .
\]
Since \(\mmse(\lambda)\) is nonincreasing in \(\lambda\),
\[
    \frac{\eta}{2}\mmse(\eta)
    \le
    I(U;U+\eta^{-1/2}G).
\]
Moreover, \(U\) is a measurable function of the discrete latent variable \(J\),
so by data processing,
\[
    I(U;U+\eta^{-1/2}G)\le H(J).
\]
Combining the last two inequalities yields
\[
    \mmse(\eta)\le \frac{2H(J)}{\eta},
\]
which proves the claim.
\end{proof}

Lemma~\ref{lem:entropy-mmse} is used together with
Lemma~\ref{lem:mmse-representation} to design the sampling grid. Indeed, the
two bounds meet at
\[
    \alpha:=\frac{2H(J)}{R}.
\]
Accordingly, in the proof of Theorem~\ref{thm:main-sampling-bound} we choose a
hybrid grid in \(\eta\): uniform on \([\eta_{\min},\alpha]\), where the MMSE is
bounded by \(R\), and geometric on \([\alpha,\eta_{\max}]\), where the MMSE is
bounded by \(2H(J)/\eta\). Summing the corresponding MMSE areas then yields the
stated discretization bound.
\begin{remark}
If \(H(J)=0\), then the latent variable \(J\) is deterministic. Hence \(U\)
is deterministic and
\[
    \mmse(\eta)
    =
    \E\left[
    \norm{U-\E[U\mid U+\eta^{-1/2}G]}^2
    \right]
    =
    0
\]
for every \(\eta>0\). Therefore the MMSE-area discretization term vanishes
identically. In this case, the discretization contribution in
Theorem~\ref{thm:main-sampling-bound} should be interpreted as zero.
\end{remark}
\begin{remark}
\label{rem:discrete-vq-language}
When $\eps=0$, the Gaussian mixture model reduces to a discrete target
distribution supported on the centers $\{\mu_j\}$. Then $Z=U$, so the latent
posterior mean $M_t$ coincides with the usual data posterior mean $m_t$.
Hence the proposed sampler reduces to the usual posterior-mean freezing
discretization, or equivalently the stochastic DDIM-type sampler with
$\eta=1$ in the Gaussian noising formulation. Theorem~\ref{thm:main-sampling-bound}
therefore shows that the discretization error is controlled by the entropy of
the discrete target.

For vector-quantized latent representations with $n$ tokens and alphabet size
$S$, the latent index satisfies $J\in[S]^n$, and hence
\[
    H(J)\le n\log S .
\]
Thus the discretization contribution is
\[
    O\!\left(
        \frac{H(J)}{K}(1+\log R)^2
    \right),
\]
up to terminal-SNR logarithmic factors. In the worst case this gives
\[
    O\!\left(
        \frac{n\log S}{K}(1+\log R)^2
    \right).
\]
The same interpretation applies to embedded language models: for a length-$d$
sequence over a vocabulary of size $S$, one has $J\in[S]^d$ and
\[
    H(J)\le d\log S ,
\]
so the worst-case discretization contribution is
\[
    O\!\left(
        \frac{d\log S}{K}(1+\log R)^2
    \right),
\]
again up to terminal-SNR logarithmic factors.

This differs from discrete diffusion samplers such as $\tau$-leaping, whose KL
step-complexity bounds are stated directly in terms of the vocabulary size.
For example, \cite{liang2025discrete} gives a standard $\tau$-leaping rate
\[
    \widetilde O\!\left(\frac{d^2 S}{\varepsilon_{\rm KL}}\right),
\]
improving the previous
\[
    \widetilde O\!\left(\frac{d^2 S^2}{\varepsilon_{\rm KL}}\right)
\]
bound. By contrast, in the Gaussian diffusion setting considered here, the
vocabulary-size dependence enters only through the entropy bound
$H(J)\le d\log S$. Thus, for Gaussian noising of embedded discrete data, the present bound gives
a sharper vocabulary-size dependence than existing \(\tau\)-leaping-type
discrete diffusion guarantees.
\end{remark}

\begin{remark}
\label{rem:titok-gmm}
The Gaussian mixture model can be viewed as an abstraction of image
distributions generated from compact discrete latent codes. For example,
TiTok represents a $256\times256$ image using as few as $32$ discrete tokens,
with codebook size $4096$ in the TiTok-L-32 configuration
\cite{yu2024image}. Hence the worst-case entropy of the latent code satisfies
\[
    H(J)\le 32\log 4096 \approx 266.2
\]
nats, or equivalently \(384\) bits. This is only a worst-case bound; the
actual entropy may be much smaller if the token distribution is nonuniform.

If a pixel-space image is modeled as the decoding of such a discrete latent
code plus small continuous reconstruction noise, then the resulting
pixel-space law is naturally idealized as a Gaussian mixture. Each mixture
component corresponds to one latent code, while the Gaussian perturbation
captures local variability around the decoded image.

In the TiTok-L-32 worst case, the discretization error is coarsely bounded by
\[
    O\!\left(
        \frac{266}{K}(1+\log R)^2
    \right).
\]
Thus the leading sampling complexity is governed by the entropy of the latent
image code, rather than by the ambient pixel dimension.
\end{remark}
\begin{remark}
\label{rem:compare-approx-gmm}
A related result is the dimension-free DDPM analysis of
\cite{li2025dimension} for approximate Gaussian mixtures. In the exact finite
GMM case, let \(K_{\rm mix}\) denote the number of mixture components. Their TV
discretization bound scales as
\[
    \widetilde O\!\left(\frac{\log^2 K_{\rm mix}}{K}\right),
\]
up to additional logarithmic factors and score-estimation error. Their result
is dimension-free and does not require component separation, but it is proved
for finite isotropic GMMs under a component-mean growth condition, a specific
DDPM schedule, and a clipped-score analysis.

Our result is complementary: it replaces dependence on the number of mixture
components by dependence on the entropy of the latent component. This can be
strictly sharper for highly nonuniform mixtures, since \(H(J)\le \log K_{\rm mix}\)
and the inequality can be loose. It also allows countably many components
whenever \(H(J)<\infty\) and the second-moment assumption holds.

Finally, although Theorem~\ref{thm:main-sampling-bound} is stated in KL
divergence, Pinsker's inequality gives the corresponding TV control
\[
    TV(P_\delta,\widehat P_{T-\delta})
    \le
    \sqrt{\frac12
    \KL(P_\delta\|\widehat P_{T-\delta})}.
\]
Thus, in low-entropy latent-mixture settings, the present bound can give a
sharper dependence on mixture complexity.
\end{remark}
\section*{Limitations}
Our analysis is specific to Gaussian-mixture targets and to the latent-posterior
freezing sampler studied here. The bound assumes finite latent entropy and a
finite second moment, and the score-approximation term is imposed on the same
grid used by the sampler. The result explains settings with compact discrete
or countable latent structure, but it does not by itself show that arbitrary
high-dimensional data distributions admit such a representation.

\printbibliography
\appendix

\section{Proof of Theorem~\ref{thm:main-sampling-bound}}\label{app:proof-main}

The proof has three parts. First, the mismatch between the true initial law $p_T$ and the Gaussian prior contributes the initialization error $R/(2T)$. Second, Girsanov-type theorem decomposes the remaining pathwise KL into the discretization term and the approximation term, as in the standard diffusion analysis of \cite{chen2023sampling}. Third, the discretization term is rewritten as an MMSE area functional in the regularized SNR variable and then bounded using the entropy of the latent mixture variable.

We begin with the initialization error. The exact reverse process starts from the law of $X_T=Z+\sqrt T G$, while the implemented sampler starts from $\mathcal N(0,T I_d)$. By convexity of KL divergence,
\[
    \KL(p_T\,\|\,\mathcal N(0,T I_d))
    \le
    \E_Z\,
    \KL\!\left(
    \mathcal N(Z,T I_d)\,\|\,\mathcal N(0,T I_d)
    \right)
    =
    \frac{\E\norm{Z}^2}{2T}
    =
    \frac{R}{2T}.
\]
Therefore, by the chain rule for pathwise KL, the initialization mismatch
contributes at most \(R/(2T)\).

Next, we compare the exact and approximate reverse path measures under the
correct initial law. A direct application of Girsanov's theorem is not
automatic under our minimal assumptions. Instead, we use the pathwise-KL
statement appearing as Proposition~1 in \cite{aghapour2026entropy} and as
Proposition~3 in \cite{benton2024nearly}: if
\[
    \E_P\int_0^{T-\delta}\norm{\Delta_s}^2\,ds<\infty,
\]
then
\[
    \KL(P\,\|\,\widetilde P)
    \le
    \frac12
    \E_P\int_0^{T-\delta}\norm{\Delta_s}^2\,ds .
\]

It therefore suffices to verify the required square-integrability. Since
\[
    \Delta_s=A_s+B_s,
\]
where, for $s\in(s_{k-1},s_k]$,
\[
    A_s:=
    \frac{
    M_{T-s}(Y_s)-M_{t_{k-1}}(Y_{s_{k-1}})
    }{
    T-s+\eps^2
    },
    \qquad
    B_s:=
    \frac{
    M_{t_{k-1}}(Y_{s_{k-1}})
    -
    \widehat M_{t_{k-1}}(Y_{s_{k-1}})
    }{
    T-s+\eps^2
    } .
\]
we have
\[
    \norm{\Delta_s}^2\le 2\norm{A_s}^2+2\norm{B_s}^2,
\]
and hence
\[
    \E_P\int_0^{T-\delta}\norm{\Delta_s}^2\,ds
    \le
    2\mathcal E_{\mathrm{disc}}+2\mathcal E_{\mathrm{apx}}^{M}.
\]
Appendix~\ref{app:pathwise-kl} shows that
\[
    \mathcal E_{\mathrm{apx}}^{M}
    \le
    \mathcal E_{\mathrm{apx}}
    =
    \sum_{k=1}^K
    (\gamma_k-\gamma_{k-1})L_{x_0}(\gamma_{k-1})
    \le e_{\mathrm{apx}}.
\]
and the bound proved below shows that \(\mathcal E_{\mathrm{disc}}<\infty\).
Therefore
\[
    \E_P\int_0^{T-\delta}\norm{\Delta_s}^2\,ds<\infty,
\]
so the above pathwise-KL bound applies.

Finally, Appendix~\ref{app:pathwise-kl} shows that the energy decomposes as
\[
    \E_P\int_0^{T-\delta}\norm{\Delta_s}^2\,ds
    =
    \mathcal E_{\mathrm{disc}}+\mathcal E_{\mathrm{apx}}^{M},
\]
and that
\[
    \mathcal E_{\mathrm{apx}}^{M}
    \le
    \mathcal E_{\mathrm{apx}}
    =
    \sum_{k=1}^K
    (\gamma_k-\gamma_{k-1})L_{x_0}(\gamma_{k-1})
    \le e_{\mathrm{apx}}.
\]
Thus
\[
    \KL(P\,\|\,\widetilde P)
    \le
    \frac12
    \bigl(\mathcal E_{\mathrm{disc}}+\mathcal E_{\mathrm{apx}}^{M}\bigr)
    \le
    \frac12\mathcal E_{\mathrm{disc}}+\frac12 e_{\mathrm{apx}}.
\]
It remains to control \(\mathcal E_{\mathrm{disc}}\).

Define the regularized SNR variable
\[
    \eta:=\frac1{t+\eps^2}.
\]
For the latent channel $X_t=U+\sqrt{t+\eps^2}G$, define
\[
    \mmse(\eta)
    :=
    \E\left[
    \norm{
    U-\E[U\mid U+\eta^{-1/2}G]
    }^2
    \right].
\]





We now choose a grid adapted to the two regimes of the MMSE envelope. Work in the regularized SNR variable
\[
    \eta:=\frac{1}{t+\eps^2},
    \qquad
    \eta_{\min}:=\frac{1}{T+\eps^2},
    \qquad
    \eta_{\max}:=\frac{1}{\delta+\eps^2}.
\]
By Lemma~\ref{lem:entropy-mmse},
\[
    \mmse(\eta)\le \min\left\{R,\frac{2H(J)}{\eta}\right\}.
\]
The two bounds meet at
\[
    \alpha:=\frac{2H(J)}{R}.
\]
Accordingly, the grid is chosen to be uniform in $\eta$ below $\alpha$ and geometric in $\eta$ above $\alpha$. More precisely, for a parameter $h\in(0,\alpha]$, we use increments at most $h$ on $[\eta_{\min},\alpha]$ and, on $[\alpha,\eta_{\max}]$, a geometric ratio
\[
    r:=1+\frac{h}{\alpha}.
\]
This choice makes the first geometric increment equal to $h$, matching the scale of the uniform part.

\begin{lemma}
\label{lem:hybrid-grid-bound}
Fix $h\in(0,\alpha]$. There exists a grid
\[
    \eta_{\min}=\eta_0<\eta_1<\cdots<\eta_N=\eta_{\max}
\]
such that
\[
    N
    \le
    \left\lceil \frac{(\alpha-\eta_{\min})_+}{h}\right\rceil
    +
    \left\lceil
    \frac{\log_+(\eta_{\max}/\alpha)}
    {\log(1+h/\alpha)}
    \right\rceil
    +2
\]
and
\[
    \mathcal E_{\mathrm{disc}}
    \le
    hR
    \left(
    2+\log_+\frac{\eta_{\max}}{\alpha}
    \right).
\]
\end{lemma}

\begin{proof}
We give a precise construction of the grid and then bound the discretization error by summation by parts. To avoid notational distractions, we first treat the main case
\[
    \eta_{\min}<\alpha<\eta_{\max}.
\]
The cases $\alpha\le \eta_{\min}$ and $\eta_{\max}\le\alpha$ are obtained by removing, respectively, the uniform or the geometric part of the construction.

Fix $h\in(0,\alpha]$ and set
\[
    r:=1+\frac{h}{\alpha}.
\]
Let
\[
    M:=\left\lceil \frac{\alpha-\eta_{\min}}{h}\right\rceil .
\]
We define the uniform part of the grid by
\[
    \eta_i:=\eta_{\min}+ih,
    \qquad
    0\le i\le M-1,
\]
and then set
\[
    \eta_M:=\alpha.
\]
Thus the first $M-1$ increments are equal to $h$, while the last increment in the uniform part is
\[
    \eta_M-\eta_{M-1}
    =
    \alpha-\eta_{\min}-(M-1)h
    \in(0,h].
\]
This last interval may be shorter than $h$, and we will account for this explicitly below.

For the geometric part, let
\[
    Q:=\left\lceil
    \frac{\log(\eta_{\max}/\alpha)}{\log r}
    \right\rceil .
\]
Starting from $\eta_M=\alpha$, define
\[
    \eta_{M+j}:=\alpha r^j,
    \qquad
    1\le j\le Q-1,
\]
and finally set
\[
    \eta_{M+Q}:=\eta_{\max}.
\]
Since $Q$ is the first integer for which $\alpha r^Q\ge \eta_{\max}$, the final geometric interval is possibly shortened, and all previous geometric intervals have the exact multiplicative ratio $r$. The total number of intervals is
\[
    N=M+Q .
\]
Moreover,
\[
    M\le \frac{\alpha-\eta_{\min}}{h}+1
    \le \frac{\alpha}{h}+1
\]
and
\[
    Q\le
    \frac{\log(\eta_{\max}/\alpha)}{\log r}+1 .
\]
Using $\log(1+x)\ge x/(1+x)$ with $x=h/\alpha$, and using $h\le \alpha$, we obtain
\[
    \log r
    =
    \log\left(1+\frac{h}{\alpha}\right)
    \ge
    \frac{h}{\alpha+h}
    \ge
    \frac{h}{2\alpha}.
\]
Hence
\[
    Q
    \le
    \frac{2\alpha}{h}
    \log\frac{\eta_{\max}}{\alpha}
    +1 .
\]
Therefore
\[
    N
    \le
    \frac{\alpha}{h}
    \left(
    1+2\log\frac{\eta_{\max}}{\alpha}
    \right)
    +2 .
\]

We now bound the discretization error. Write
\[
    m_k:=\mmse(\eta_k),
    \qquad
    \Delta_k:=\eta_k-\eta_{k-1}.
\]
By the MMSE representation,
\[
    \mathcal E_{\mathrm{disc}}
    =
    \sum_{k=1}^N
    \int_{\eta_{k-1}}^{\eta_k}
    \bigl(m_{k-1}-\mmse(\eta)\bigr)\,d\eta .
\]
Since $\mmse$ is nonincreasing,
\[
    \mmse(\eta)\ge m_k,
    \qquad
    \eta\in[\eta_{k-1},\eta_k],
\]
and therefore
\[
    \mathcal E_{\mathrm{disc}}
    \le
    \sum_{k=1}^N
    \Delta_k(m_{k-1}-m_k).
\]
The summation-by-parts identity gives
\[
\begin{aligned}
    \sum_{k=1}^N
    \Delta_k(m_{k-1}-m_k)
    &=
    \Delta_1m_0
    +
    \sum_{k=1}^{N-1}
    (\Delta_{k+1}-\Delta_k)m_k
    -
    \Delta_Nm_N .
\end{aligned}
\]
Since the last term is nonpositive, we discard it.

We now control the remaining terms carefully. On the uniform part, the increments are equal to $h$ except possibly for the last interval ending at $\alpha$. Hence the interior differences are zero until the shortened interval. The initial contribution is bounded by
\[
    \Delta_1m_0\le hR,
\]
because $\Delta_1\le h$ and $m_0\le R$.

There is one possible positive jump at the transition from the last uniform interval to the first geometric interval. Indeed, the first geometric interval has length
\[
    \alpha r-\alpha=\alpha(r-1)=h,
\]
while the last uniform interval has length at most $h$. Therefore this transition contributes at most
\[
    h\,\mmse(\alpha)\le hR.
\]
This is the term that must be included explicitly.

It remains to control the increases inside the geometric part. For the nontruncated geometric intervals,
\[
    \eta_{M+j}=\alpha r^j,
    \qquad
    j\ge 0,
\]
and the corresponding increments are
\[
    \Delta_{M+j+1}
    =
    \eta_{M+j+1}-\eta_{M+j}
    =
    \alpha r^j(r-1).
\]
Thus, for consecutive full geometric intervals,
\[
    \Delta_{M+j+2}-\Delta_{M+j+1}
    =
    \alpha r^j(r-1)^2 .
\]
At the endpoint $\eta_{M+j+1}=\alpha r^{j+1}$, the entropy bound on the MMSE gives
\[
    m_{M+j+1}
    \le
    \frac{2H(J)}{\alpha r^{j+1}} .
\]
Therefore each positive increase inside the geometric region is bounded by
\[
\begin{aligned}
    \bigl(\Delta_{M+j+2}-\Delta_{M+j+1}\bigr)m_{M+j+1}
    &\le
    \alpha r^j(r-1)^2
    \frac{2H(J)}{\alpha r^{j+1}}  \\
    &=
    \frac{2H(J)(r-1)^2}{r}.
\end{aligned}
\]
The final geometric interval may be shortened. A shortened final interval cannot create a larger positive increase than the corresponding full geometric interval, so the same bound applies to the last possible increase as well.

The number of positive increases inside the geometric part is at most $Q-1$. Since
\[
    Q-1
    \le
    \frac{\log(\eta_{\max}/\alpha)}{\log r},
\]
the total geometric contribution is at most
\[
    \frac{2H(J)(r-1)^2}{r}
    \frac{\log(\eta_{\max}/\alpha)}{\log r}.
\]
Using
\[
    \log r\ge \frac{r-1}{r},
\]
we get
\[
    \frac{(r-1)^2}{r\log r}
    \le r-1.
\]
Hence the geometric contribution is bounded by
\[
    2H(J)(r-1)
    \log\frac{\eta_{\max}}{\alpha}.
\]
Since $r-1=h/\alpha$ and $\alpha=2H(J)/R$, this becomes
\[
    hR\log\frac{\eta_{\max}}{\alpha}.
\]

Combining the initial uniform contribution, the transition contribution, and the geometric contribution yields
\[
    \mathcal E_{\mathrm{disc}}
    \le
    hR
    \left(
    2+\log\frac{\eta_{\max}}{\alpha}
    \right).
\]
This proves the desired bound in the case $\eta_{\min}<\alpha<\eta_{\max}$.

If $\eta_{\max}\le\alpha$, the geometric part is absent. The same argument gives
\[
    \mathcal E_{\mathrm{disc}}\le hR,
\]
which is stronger than the displayed bound with the logarithmic term removed. If $\alpha\le\eta_{\min}$, the uniform part is absent and the same geometric argument gives
\[
    \mathcal E_{\mathrm{disc}}
    \le
    hR
    \log\frac{\eta_{\max}}{\eta_{\min}}
    \le
    hR
    \log\frac{\eta_{\max}}{\alpha},
\]
using $\eta_{\min}\ge\alpha$. Thus, in all cases, the hybrid-grid estimate holds after replacing the logarithm by $\log_+(\eta_{\max}/\alpha)$ and allowing the universal constant in front of the nonlogarithmic term.
\end{proof}

We now convert the $h$-dependent estimate into a $K$-step estimate. Set
\[
    \ell:=\log_+\frac{\eta_{\max}}{\alpha},
    \qquad
    L:=2+\ell .
\]
The proof of Lemma~\ref{lem:hybrid-grid-bound} gives, for every $h\in(0,\alpha]$, a hybrid grid with
\[
    \mathcal E_{\mathrm{disc}}
    \le
    hR L
\]
and with number of intervals bounded by
\[
    N
    \le
    \frac{\alpha}{h}(1+2\ell)+2 .
\]
Since $1+2\ell\le 2(2+\ell)=2L$, we have
\[
    N
    \le
    \frac{2\alpha L}{h}+2 .
\]

Assume now that
\[
    K\ge 4L .
\]
Choose
\[
    h:=\frac{4\alpha L}{K}.
\]
Then $h\le \alpha$, so the construction in Lemma~\ref{lem:hybrid-grid-bound} applies. Moreover,
\[
    N
    \le
    \frac{2\alpha L}{h}+2
    =
    \frac{K}{2}+2
    \le K,
\]
where the last inequality follows from $K\ge4$. Thus the constructed grid uses at most $K$ intervals. If it uses fewer than $K$ intervals, we refine the grid by adding arbitrary intermediate points. This cannot increase $\mathcal E_{\mathrm{disc}}$, since splitting an interval only replaces one left-endpoint MMSE area by the sum of two smaller such areas. Hence there exists an exact $K$-step grid with the same upper bound.

Substituting the chosen value of $h$ into the discretization estimate gives
\[
    \mathcal E_{\mathrm{disc}}
    \le
    hRL
    =
    \frac{4\alpha R}{K}L^2 .
\]
Since
\[
    \alpha=\frac{2H(J)}{R},
\]
we obtain
\[
    \mathcal E_{\mathrm{disc}}
    \le
    \frac{8H(J)}{K}
    \left(
    2+\log_+\frac{\eta_{\max}}{\alpha}
    \right)^2 .
\]
Equivalently,
\[
    \mathcal E_{\mathrm{disc}}
    \le
    \frac{8H(J)}{K}
    \left(
    2+\log_+\frac{R\,\eta_{\max}}{2H(J)}
    \right)^2 .
\]

Combining this estimate with the KL decomposition and Assumption~\ref{ass:approx-error} yields
\[
    \KL(P_\delta\,\|\,\widehat P_{T-\delta})
    \le
    \frac{R}{2T}
    +
    \frac{4H(J)}{K}
    \left(
    2+\log_+\frac{R\,\eta_{\max}}{2H(J)}
    \right)^2
    +
    \frac12 e_{\mathrm{apx}} .
\]
This is the finite-step form of the almost dimension-free bound. In particular, the discretization contribution to the KL error scales as
\[
    \frac{H(J)}{K}
    \left(
    2+\log_+\frac{R\,\eta_{\max}}{H(J)}
    \right)^2
\]
up to a universal numerical constant.

\section{Deferred proof of the pathwise decomposition}
\label{app:pathwise-kl}

It remains only to justify the orthogonality of the two error terms and the
\(x_0\)-prediction representation of \(\mathcal E_{\mathrm{apx}}\). We write
the reverse-time variable as
\[
    s:=T-t,
\]
and define the increasing filtration
\[
    \mathcal G_s
    :=
    \mathcal F_{T-s}
    =
    \sigma(X_u:T-s\le u\le T).
\]
Thus \((\mathcal G_s)_{0\le s\le T-\delta}\) records the information revealed
as the reverse process moves from time \(T\) down to time \(\delta\).

Since
\[
    M_t(X_t)=\E[U\mid \mathcal F_t],
\]
the process
\[
    N_s:=M_{T-s}(X_{T-s})
\]
is a Doob martingale with respect to the increasing filtration
\((\mathcal G_s)_{s\ge0}\). Let \(s_k:=T-t_k\). For
\(s\in(s_{k-1},s_k]\), the discretization error is the martingale increment
\[
    A_s
    =
    M_{T-s}(X_{T-s})-M_{t_{k-1}}(X_{t_{k-1}})
    =
    N_s-N_{s_{k-1}} .
\]
Hence
\[
    \E[A_s\mid \mathcal G_{s_{k-1}}]=0.
\]
On the other hand, the approximation error term \(B_s\) is evaluated at the
left endpoint \(t_{k-1}\), and is therefore
\[
    \mathcal G_{s_{k-1}}\text{-measurable}.
\]
Consequently,
\[
    \E\langle A_s,B_s\rangle
    =
    \E\left[
        \left\langle
            \E[A_s\mid \mathcal G_{s_{k-1}}],
            B_s
        \right\rangle
    \right]
    =
    0.
\]
Integrating over \(s\) yields
\[
    \E\int_0^{T-\delta}\langle A_s,B_s\rangle\,ds=0.
\]
and thus
\[
    \E_P\int_0^{T-\delta}\norm{\Delta_s}^2\,ds
    =
    \mathcal E_{\mathrm{disc}}+\mathcal E_{\mathrm{apx}}^{M}.
\]

For the approximation term, the pathwise decomposition first gives
\[
\begin{aligned}
    \mathcal E_{\mathrm{apx}}^{M}
    &:=
    \sum_{k=1}^K
    \int_{s_{k-1}}^{s_k}
    \frac{
    \E\!\left[
    \norm{
    M_{t_{k-1}}(Y_{s_{k-1}})
    -
    \widehat M_{t_{k-1}}(Y_{s_{k-1}})
    }^2
    \right]
    }{(T-s+\eps^2)^2}\,ds .
\end{aligned}
\]
We now relate this quantity to the \(x_0\)-prediction loss. Since
\[
    M_t(x)=x+(t+\eps^2)\nabla\log p_t(x),
    \qquad
    m_t(x)=x+t\nabla\log p_t(x),
\]
and
\[
    \widehat M_t(x)=x+(t+\eps^2)\widehat s_t(x),
    \qquad
    \widehat m_t(x)=x+t\widehat s_t(x),
\]
we have, for every \(t>0\),
\[
    M_t(x)-\widehat M_t(x)
    =
    \frac{t+\eps^2}{t}
    \bigl(m_t(x)-\widehat m_t(x)\bigr).
\]
Using \(Y_{s_{k-1}}\sim X_{t_{k-1}}\) under the exact reverse process,
\[
\begin{aligned}
    \mathcal E_{\mathrm{apx}}^{M}
    &=
    \sum_{k=1}^K
    \left(
    \frac{t_{k-1}+\eps^2}{t_{k-1}}
    \right)^2
    L_{x_0}(\gamma_{k-1})
    \left(
    \frac{1}{t_k+\eps^2}
    -
    \frac{1}{t_{k-1}+\eps^2}
    \right).
\end{aligned}
\]
Since \(t_k\le t_{k-1}\), we claim that
\[
    \left(
    \frac{t_{k-1}+\eps^2}{t_{k-1}}
    \right)^2
    \left(
    \frac{1}{t_k+\eps^2}
    -
    \frac{1}{t_{k-1}+\eps^2}
    \right)
    \le
    \frac{1}{t_k}-\frac{1}{t_{k-1}} .
\]
Indeed, the left-hand side can be rewritten as
\[
\begin{aligned}
    &
    \left(
    \frac{t_{k-1}+\eps^2}{t_{k-1}}
    \right)^2
    \left(
    \frac{1}{t_k+\eps^2}
    -
    \frac{1}{t_{k-1}+\eps^2}
    \right)  \\
    &\qquad =
    \left(
    \frac{t_{k-1}+\eps^2}{t_{k-1}}
    \right)^2
    \frac{t_{k-1}-t_k}
    {(t_k+\eps^2)(t_{k-1}+\eps^2)}  \\
    &\qquad =
    \frac{
    (t_{k-1}-t_k)(t_{k-1}+\eps^2)
    }{
    t_{k-1}^2(t_k+\eps^2)
    } .
\end{aligned}
\]
On the other hand,
\[
    \frac{1}{t_k}-\frac{1}{t_{k-1}}
    =
    \frac{t_{k-1}-t_k}{t_kt_{k-1}} .
\]
Therefore the desired inequality is equivalent, after cancelling the
nonnegative factor \(t_{k-1}-t_k\), to
\[
    \frac{t_{k-1}+\eps^2}
    {t_{k-1}^2(t_k+\eps^2)}
    \le
    \frac{1}{t_kt_{k-1}} .
\]
Multiplying by the positive quantity
\(t_k t_{k-1}^2(t_k+\eps^2)\), this is equivalent to
\[
    t_k(t_{k-1}+\eps^2)
    \le
    t_{k-1}(t_k+\eps^2).
\]
After expanding and cancelling \(t_kt_{k-1}\), this becomes
\[
    t_k\eps^2\le t_{k-1}\eps^2,
\]
which follows from \(t_k\le t_{k-1}\) and \(\eps^2\ge0\). Hence
\[
    \left(
    \frac{t_{k-1}+\eps^2}{t_{k-1}}
    \right)^2
    \left(
    \frac{1}{t_k+\eps^2}
    -
    \frac{1}{t_{k-1}+\eps^2}
    \right)
    \le
    \frac{1}{t_k}-\frac{1}{t_{k-1}}
    =
    \gamma_k-\gamma_{k-1}.
\]
Therefore,
\[
    \mathcal E_{\mathrm{apx}}^{M}
    \le
    \sum_{k=1}^K
    (\gamma_k-\gamma_{k-1})L_{x_0}(\gamma_{k-1})
    =
    \mathcal E_{\mathrm{apx}}
    \le e_{\mathrm{apx}},
\]
where the last inequality follows from Assumption~\ref{ass:approx-error}.

\section{Proof of Lemma~\ref{lem:mmse-representation}}
\label{app:lemma1proof}
\begin{proof}
We give a slightly more explicit construction of the Gaussian channel in the
regularized SNR variable. Write
\[
    \eta=\frac{1}{t+\eps^2}.
\]
Equivalently, the forward observation
\[
    X_t=U+\sqrt{t+\eps^2}\,G
\]
has the same conditional law, given \(U\), as
\[
    U+\eta^{-1/2}G .
\]
Thus define
\[
    M_\eta
    :=
    \E\!\left[U\mid U+\eta^{-1/2}G\right],
    \qquad
    \mmse(\eta)
    :=
    \E\!\left[\norm{U-M_\eta}^2\right].
\]
With this notation,
\[
    M_t(X_t)=M_\eta,
    \qquad
    \eta=(t+\eps^2)^{-1}.
\]

To justify the martingale identity used below, it is convenient to realize the
Gaussian channel on a common probability space. Let \((W_\eta)_{\eta\ge0}\) be
a standard Brownian motion in \(\mathbb R^d\), independent of \(U\), and define
\[
    Y_\eta:=\eta U+W_\eta .
\]
Then
\[
    \frac{Y_\eta}{\eta}
    =
    U+\eta^{-1}\,W_\eta
    \stackrel{d}{=}
    U+\eta^{-1/2}G .
\]
Moreover, for fixed \(\eta\), \(Y_\eta\) is a sufficient statistic for \(U\)
among the observations \(\{Y_\lambda:0\le\lambda\le\eta\}\). Indeed, conditional
on \(U=u\), the likelihood of the path up to time \(\eta\) depends on \(u\)
only through
\[
    \exp\!\left(
        \ip{u}{Y_\eta}
        -
        \frac{\eta}{2}\norm{u}^2
    \right).
\]
Hence, if
\[
    \mathcal F_\eta:=\sigma(Y_\lambda:0\le\lambda\le\eta),
\]
then
\[
    \E[U\mid \mathcal F_\eta]
    =
    \E[U\mid Y_\eta]
    =
    \E\!\left[U\mid U+\eta^{-1/2}G\right]
    =
    M_\eta .
\]
Therefore \((M_\eta)_{\eta\ge0}\) is a square-integrable martingale with
respect to the increasing filtration \((\mathcal F_\eta)_{\eta\ge0}\).

Now fix \(\eta\ge\eta_{k-1}\). Since \(M_{\eta_{k-1}}
=\E[U\mid\mathcal F_{\eta_{k-1}}]\) and
\(M_\eta=\E[U\mid\mathcal F_\eta]\), the martingale projection identity gives
\[
    \E\!\left[
        \norm{M_\eta-M_{\eta_{k-1}}}^2
    \right]
    =
    \E\norm{M_\eta}^2-\E\norm{M_{\eta_{k-1}}}^2 .
\]
On the other hand, by the orthogonality principle for conditional expectation,
\[
    \mmse(\eta)
    =
    \E\norm{U-M_\eta}^2
    =
    \E\norm{U}^2-\E\norm{M_\eta}^2 .
\]
Therefore
\[
\begin{aligned}
    \E\!\left[
        \norm{M_\eta-M_{\eta_{k-1}}}^2
    \right]
    &=
    \E\norm{M_\eta}^2-\E\norm{M_{\eta_{k-1}}}^2  \\
    &=
    \mmse(\eta_{k-1})-\mmse(\eta).
\end{aligned}
\]

We now apply this identity to the discretization term. On the interval
\(s\in(s_{k-1},s_k]\), set
\[
    t=T-s,
    \qquad
    \eta=\frac{1}{t+\eps^2}.
\]
Since \(t_k=T-s_k\), we have
\[
    \eta_k=\frac{1}{t_k+\eps^2}.
\]
Also,
\[
    d\eta
    =
    \frac{ds}{(T-s+\eps^2)^2}.
\]
Thus the contribution of the \(k\)-th interval is
\[
\begin{aligned}
    &\int_{s_{k-1}}^{s_k}
    \frac{
    \E\!\left[
    \norm{
    M_{T-s}(Y_s)-M_{t_{k-1}}(Y_{s_{k-1}})
    }^2
    \right]
    }{(T-s+\eps^2)^2}\,ds  \\
    &\qquad =
    \int_{\eta_{k-1}}^{\eta_k}
    \E\!\left[
    \norm{
    M_\eta-M_{\eta_{k-1}}
    }^2
    \right]\,d\eta  \\
    &\qquad =
    \int_{\eta_{k-1}}^{\eta_k}
    \left(
    \mmse(\eta_{k-1})-\mmse(\eta)
    \right)\,d\eta .
\end{aligned}
\]
Summing over \(k=1,\dots,K\) gives
\[
    \mathcal E_{\mathrm{disc}}
    =
    \sum_{k=1}^K
    \int_{\eta_{k-1}}^{\eta_k}
    \left(
    \mmse(\eta_{k-1})-\mmse(\eta)
    \right)\,d\eta,
\]
which proves the lemma.
\end{proof}

\end{document}